\documentclass{article}

\usepackage[dblblindworkshop, final]{neurips_2025}
\usepackage{subcaption}
\usepackage{multirow}

\usepackage[utf8]{inputenc} 
\usepackage[T1]{fontenc}    
\usepackage{hyperref}       
\usepackage{url}            
\usepackage{booktabs}       
\usepackage{amsfonts}       
\usepackage{nicefrac}       
\usepackage{microtype}      
\usepackage{xcolor}         
\hypersetup{
    colorlinks,
    linkcolor={red!50!black},
    citecolor={blue!50!black},
    urlcolor={blue!80!black}
}
\usepackage[ruled,vlined,linesnumbered]{algorithm2e}
\usepackage{amsmath}
\DeclareMathOperator*{\argmax}{arg\,max}

\usepackage{bbm}
\usepackage[colorinlistoftodos,prependcaption,textsize=tiny]{todonotes}

\usepackage{amsmath,amssymb,amsthm}
\usepackage{thmtools}
\usepackage{thm-restate}

\usepackage{amssymb}
\usepackage{mathtools}
\usepackage{amsthm}
\theoremstyle{plain}
\newtheorem{theorem}{Theorem}[section]

\newtheorem{lemma}[theorem]{Lemma}

\workshoptitle{Efficient Reasoning}

\title{\texttt{RoiRL}: Efficient, Self-Supervised Reasoning with Offline Iterative Reinforcement Learning}

\author{%
  Aleksei Arzhantsev \\
  Criteo AI Lab, Ecole Polytechnique\\
  Paris, France\\
  \texttt{a.arzhantsev@criteo.com} \\
  \And
    Otmane Sakhi \\
  Criteo AI Lab\\
  Paris, France\\
  \texttt{o.sakhi@criteo.com} \\
  \And
  Flavian Vasile \\
  Criteo AI Lab\\
  Paris, France\\
  \texttt{f.vasile@criteo.com} \\
}

\begin{document}

\maketitle

\begin{abstract}
Reinforcement learning (RL) is central to improving reasoning in large language models (LLMs) but typically requires ground-truth rewards. Test-Time Reinforcement Learning (\texttt{TTRL}) removes this need by using majority-vote rewards, but relies on heavy online RL and incurs substantial computational cost. We propose \texttt{RoiRL}: \textbf{R}easoning with \textbf{o}ffline \textbf{i}terative \textbf{R}einforcement \textbf{L}earning, a family of lightweight offline learning alternatives that can target the same regularized optimal policies. Unlike \texttt{TTRL}, \texttt{RoiRL} eliminates the need to maintain a reference model and instead optimizes \emph{weighted log-likelihood} objectives, enabling stable training with significantly lower memory and compute requirements. Experimental results show that \texttt{RoiRL} trains to $2.5\times$ faster and consistently outperforms \texttt{TTRL} on reasoning benchmarks, establishing a scalable path to self-improving LLMs without labels.
\end{abstract}

\section{Introduction}


Reasoning \cite{reasoning} is at the core of large language model (LLM) capabilities, improving performance on mathematical problem solving \cite{reason_math}, commonsense inference \cite{reasoning, commonsense}, and agentic applications \cite{react}. Recent advances have demonstrated that reasoning ability can be enhanced not only by scaling model size and data but also by explicitly training models to generate and evaluate chains of thought \cite{deepseekr1}. Reinforcement learning (RL) \cite{Sutton} has played a particularly important role in this direction: RL aligns models generations with outcome quality, improving their ability to solve complex tasks.

However, RL-based approaches require access to \emph{ground-truth rewards}, mostly in the form of correctness labels (e.g., for math problems). This reliance can limit their scalability, since ground-truth supervision is costly and often unavailable. To circumvent this bottleneck, recent work has leveraged \emph{majority-vote} as a weak supervision signal: instead of relying on external labels, the model itself generates multiple candidate solutions, and majority voting \cite{selfconsistency} is used to estimate correctness. This idea has proven highly effective at inference time, where increasing the number of sampled solutions substantially improves accuracy through test-time scaling \cite{selfconsistency}.

Building on this observation, Test-Time Reinforcement Learning (\texttt{TTRL}) \cite{ttrl} has been proposed as a mechanism for turning majority-vote feedback, originally used in test time scaling, into a training signal. By repeatedly generating candidate chains of thought (CoT) \cite{reasoning} with their respective answers, evaluating them with majority voting, and updating the model parameters online, \texttt{TTRL} enables reasoning improvement without ground-truth labels. Empirically, this approach has demonstrated strong gains on reasoning benchmarks, validating the potential of self-generated feedback.


Despite its promise, \texttt{TTRL} faces two critical limitations. First, it is \emph{computationally expensive}. The method requires maintaining a reference model and computing its logits at every training step. Combined with repeated chain-of-thought (CoT) sampling during training, this quickly saturates memory and makes the approach increasingly difficult to scale to larger models or longer runs. Second, its \emph{online nature introduces instability}. Performance is highly sensitive to hyperparameter choices, as also reported in \cite{ttrl}. These issues make \texttt{TTRL} challenging to deploy in practice and limit its applicability as a general recipe for scalable reasoning improvements.

Inspired by offline RL approaches \cite{offline_rl2, offline_rl1}, we introduce \texttt{RoiRL} (\textbf{R}easoning with \textbf{o}ffline \textbf{i}terative \textbf{R}einforcement \textbf{L}earning), a lightweight alternative that preserves the benefits of self-generated rewards while overcoming the limitations of \texttt{TTRL}. Our method optimizes simple \emph{weighted log-likelihood objectives} in an iterative offline loop, eliminating the need for online RL or maintaining a reference model. This design improves stability, reduces memory overhead, and scales efficiently with model size. On small-scale models and modest compute budgets, \texttt{RoiRL} trains faster and more efficiently, while consistently surpassing \texttt{TTRL} across reasoning benchmarks.

\textbf{Contributions.} We introduce a family of \emph{offline weighted log-likelihood} objectives, that can target and solve the same underlying problem of \texttt{TTRL} without requiring online RL, nor maintaining a reference model. We demonstrate empirically that \texttt{RoiRL}, which builds on these simple objectives, achieves superior performance and scalability, offering a practical path toward self-improving LLMs without reliance on true labels.

\section{Preliminaries}

We assume access to a strong base LLM, typically pre-trained or instruction-tuned , which we denote by the policy $\pi_{0} = \pi_{\theta_0}$, . Given a prompt $x$, the model generates a chain-of-thought $c$ leading to the answer $y$, sampled as $\{c, y\} \sim \pi_0(\cdot \mid x)$. Alongside this model, we consider a collection of reasoning tasks for which ground-truth answers are unavailable. These tasks are represented as a prompt dataset $\mathcal{P}_n = \{x_i\}_{i \in [n]}$, where each $x_i$ corresponds to a question or input prompt and $n$ is the dataset size. Crucially, the dataset contains no labels, reflecting the realistic setting where large collections of problems are readily available, while solutions are not. 

\textbf{Test Time Reinforcement Learning.} Given $\mathcal{P}_n$,  \texttt{TTRL} \cite{ttrl} provides a learning algorithm to improve the reasoning abilities of $\pi_{\theta_0}$. For each reasoning task $x_i$, \texttt{TTRL} attributes an approximate ground truth label $\tilde{y}_i$ to $x_i$ on the fly. For each optimization step $t$, $\pi_{\theta_t}$ generates $k > 2$ candidates $\{c^\ell_i, y^\ell_i \}_{\ell \in [k]}$, defined with their respective CoT and answers. The answers $\{y^\ell_i \}_{\ell \in [k]}$ are used to define the approximate ground truth label $\tilde{y}_i$ as the label with the majority vote i.e. $\tilde{y}^k_i(\theta_t) = \texttt{maj}_{\ell \in [k]}(y^\ell_i)$. Rewards are then naturally attributed to the generated candidates and constructed as $\tilde{r}_k(y, x_i, \theta_t) = \mathbbm{1}\left[y = \tilde{y}^k_i(\theta_t)\right]$, augmenting data with rewards and enabling the use of RL algorithms to train  and optimize the parametrised LLM policy $\pi_\theta$. For instance, \texttt{TTRL} optimizes the KL regularized expected reward using GRPO \cite{grpo}:
\begin{align}
\label{eq:kl_regularised}
    \max_\theta\, \left\{\sum_{i = 1}^n \mathbb{E}_{(c, y)  \sim \pi_\theta(\cdot|x_i)} \left[\tilde{r}_k(y, x_i, \theta) \right] - \beta \operatorname{KL}(\pi_\theta, \pi_0|x_i) \right\}\,.
\end{align}
with $\operatorname{KL}(\pi_\theta, \pi_0|x)$ the KL divergence between $\pi_\theta(\cdot|x)$ and $\pi_0(\cdot|x)$, and $\beta > 0$ a regularization parameter, solving for an LLM that optimizes for the consistency of the generated answer while staying close to the original model. 

\textbf{Non-stationary rewards.} The particularity of the optimization problem in Equation~\eqref{eq:kl_regularised} is that the rewards are non-stationary and depend on the current policy $\pi_\theta$ we are optimizing. In step $t$, the reward of an answer $y$ is positive when it matches the majority vote  at $k$: $\tilde{y}^k(\theta_t)$. This means that the reward shifts when the majority vote changes in the optimization. This subtlety makes this approach differ from only distilling the majority voter back into the model.



\section{Self Supervised Reasoning with Iterative Offline Reinforcement Learning}

\texttt{TTRL} optimizes the KL-regularized reward maximization objective of Equation~\eqref{eq:kl_regularised} using GRPO \cite{grpo} in an online setting. While effective, this procedure is computationally demanding: it requires maintaining a reference model in memory, repeatedly sampling potentially long answers during training, and computing their logits under both current and reference policies. This saturates GPU memory and limit scalability. In addition, the reliance on online RL makes the method highly sensitive to hyperparameter choices, leading to instability and unreliable performance in practice \citep{ttrl}. 

This raises a natural question: \emph{can we achieve the same objective with a procedure as simple and stable as supervised fine-tuning?} Building on offline RL methods \cite{offline_rl1}, we answer affirmatively with \texttt{RoiRL} (\textbf{R}easoning with \textbf{o}ffline \textbf{i}terative \textbf{R}einforcement \textbf{L}earning), an iterative, offline approach that address these limitations. In each iteration $m \ge 1$, \texttt{RoiRL} alternates between two steps:

\textbf{(1) Generation:} From the current policy $\pi_{m-1}$, we sample $k$ candidate solutions $\{c^\ell_i, y^\ell_i\}_{\ell \in [k]}$ for each prompt $x_i \in \mathcal{P}$. These candidates are scored with a majority-vote reward, $\tilde{r}_k(y^\ell_i, x_i, \theta_{m-1}) = \mathbbm{1}\left[y^\ell_i = \tilde{y}^k_i(\theta_{m-1})\right]$ with $\tilde{y}^k_i(\theta_{m-1}) = \texttt{maj}_{\ell \in [k]}(y^\ell_i)$. This produces an offline dataset:
    $$\mathcal{D}_{m-1} = \left\{x_i, \left\{c^\ell_i, y^\ell_i, \tilde{r}_k(y^\ell_i, x_i, \theta_{m-1})\right\}_{\ell \in [k]} \right\}_{i \in [n]}\,$$
\textbf{(2) Offline Update:} Using $\mathcal{D}_{m-1}$, we approximate and solve the weighted log-likelihood objective
\begin{align}\label{eq:roirl}
    \theta_{m} &= \argmax_\theta \left\{\sum_{i = 1}^n\mathbb{E}_{ (c,y) \sim \pi_{m-1}(\cdot|x_i)} \left[g_m(\tilde{r}_k(y,x_i, \theta_{m-1})) \log \pi_\theta(c,y|x_i) \right] \right\}\,,
\end{align}
with $g_m: \mathbb{R} \rightarrow \mathbb{R}$ an increasing reward transform. We then update the policy as $\pi_m \leftarrow \pi_{\theta_m}$.

The resulting optimization routine is summarized in Algorithm~\ref{alg:RoiRL}. Equation~\eqref{eq:roirl} can be interpreted as a \emph{weighted supervised fine-tuning loss} on generated answers, in contrast to the unstable online updates of \texttt{TTRL}. \texttt{RoiRL} is more stable and alleviates the need for $\pi_0$, making it significantly more scalable. 

\RestyleAlgo{ruled}
\begin{restatable}{algorithm}{alg} 
\caption{\textbf{R}easoning with \textbf{o}ffline \textbf{i}terative \textbf{R}einforcement \textbf{L}earning (\texttt{RoiRL})}\label{alg:RoiRL}
\textbf{Input}: Policy $\pi_\theta$, reward transforms $(g_m)_{m\in\mathbb{N}}$, prompt dataset $\mathcal{P} = \{x_i\}_{i \in [n]}$, number of candidates $k$. \\
\textbf{Initialize}: $\theta=\theta_0$, $\pi_0=\pi_{\theta_0}$ \\
\For{$m=1,2,\dots$}{
Construct offline dataset $\mathcal{D}_{m-1}$ with $\pi_{m-1}$ \\
Update parameters $\theta_m$ by solving Equation~\eqref{eq:roirl} \\
Set $\pi_m \leftarrow \pi_{\theta_m}$
}
\end{restatable}

\textbf{Connection to KL-Regularized Objectives.} Unlike \texttt{TTRL}, which explicitly enforces KL regularization, \texttt{RoiRL} leverages a sequence $(g_m)_{m\in\mathbb{N}}$ of reward transform $g_m: \mathbb{R} \to \mathbb{R}\,, \forall m$ that implicitly control the reward influence. At iteration $m$, the analytical solution of Algorithm~\ref{alg:RoiRL} takes the form:
\begin{align}
    \forall (c,y), x\,, \quad \pi_m(c,y|x) \propto \left(\prod_{j =1}^m g_j(\tilde{r}_k(y,x, \theta_{j-1})\right) \pi_0(c,y|x)\,.
\end{align}
For example, choosing $g_j(r) = g_\beta(r) = \exp(r/\beta)\,, \forall j$ yields
\begin{align*}
    \forall (c,y), x\,, \quad \pi_m(c,y|x) \propto \exp\left(\frac{1}{\beta}\sum_{j =1}^m \tilde{r}_k(y,x, \theta_{j-1})\right) \pi_0(c,y|x)\,,
\end{align*}
which closely mirrors the closed-form solution of KL-regularized RL objectives widely used in preference alignment \cite{rlhf, dpo} and reasoning \cite{grpo} for example. A proof of the analytical solution is provided in Appendix~\ref{app:solution}. In particular, the following proposition connects \texttt{RoiRL} and \texttt{TTRL}.
\begin{restatable}{proposition}{prop}
   For any $\beta > 0$, there exists a choice of the reward transforms $(g_m)_{m\in\mathbb{N}}$ such that Equation~\eqref{eq:kl_regularised} and Algorithm~\ref{alg:RoiRL} admit the same solution. 
\end{restatable}

This result is proven in details in Appendix~\ref{app:equivalence}, and shows that \texttt{RoiRL} can target the same theoretical objective as \texttt{TTRL}, while being more stable, scalable, and practically implementable. Moreover, by flexibly choosing $g$, and thus controlling the reward influence on the updates, \texttt{RoiRL} extends beyond \texttt{TTRL} to encompass a broader family of objectives, including known regularized objectives \cite{beyond_kl}.

\section{Experiments}

\begin{table*}[t]
  \centering
  \caption{Results are reported for training on unlabeled problems from MATH500 \emph{Train} and evaluating on all datasets. \texttt{RoiRL} outperforms \texttt{TTRL} in most cases. For each decoding strategy, the second-best result is underlined, the best result is in bold, and marked with $\star$ when it beats $\pi_0$ with $\texttt{maj}_{128}$.}
  \label{tab:results}
  \resizebox{\textwidth}{!}{%
  \begin{tabular}{cl|lccc|cccc|cccc}
    \toprule
    Decode & Model & \multicolumn{4}{c}{Qwen2.5-Math-1.5B} & \multicolumn{4}{c}{Phi4-mini-reasoning-4B} & \multicolumn{4}{c}{Llama-3.2-3B-Instruct} \\
    \cmidrule(lr){3-6} \cmidrule(lr){7-10} \cmidrule(lr){11-14}
     & & \multicolumn{2}{c}{MATH500} & AMC & AIME & \multicolumn{2}{c}{MATH500} & AMC & AIME & \multicolumn{2}{c}{MATH500} & AMC & AIME \\
    \cmidrule(lr){3-4} \cmidrule(lr){7-8} \cmidrule(lr){11-12}
     & & \emph{Train} & \emph{Test} & & & \emph{Train} & \emph{Test} & & & \emph{Train} & \emph{Test} & & \\
    \midrule
    
    \multirow{4}{*}{$\texttt{maj}_{1}$} 
      & \texttt{Base} ($\pi_0$) & $0.244$ & $0.239$ & $0.170$ & $0.036$ & $0.210$ & $0.160$ & $0.071$ & $0.000$ & $0.256$ & $0.295$ & $0.141$ & $0.050$ \\
      & \texttt{TTRL} & $0.307$ & $0.298$ & $0.214$ & $0.026$ & $0.272$ & $0.225$ & $0.090$ & $0.000$ & $0.361$ & $\boldsymbol{0.394}$ & $\underline{0.159}$ & $0.043$ \\
      & \texttt{RoiRL} $g_I$ & $\boldsymbol{0.686}$ & $\underline{0.587}$ & $\underline{0.337}$ & $\boldsymbol{0.083}$ & $\boldsymbol{0.660}^\star$ & $\boldsymbol{0.511}$ & $\boldsymbol{0.246}$ & $\boldsymbol{0.016}^\star$ & $\underline{0.395}$ & $\underline{0.376}$ & $\boldsymbol{0.198}$ & $\boldsymbol{0.060}$ \\
      & \texttt{RoiRL} $g_\beta$ & $\underline{0.670}$ & $\boldsymbol{0.604}$ & $\boldsymbol{0.340}$ & $\underline{0.070}$ & $\underline{0.533}$ & $\underline{0.344}$ & $\underline{0.125}$ & $0.000$ & $\boldsymbol{0.487}$ & $0.256$ & $0.090$ & $0.020$ \\
    \midrule
    
    \multirow{4}{*}{$\texttt{maj}_{10}$} 
      & \texttt{Base} ($\pi_0$) & $0.572$ & $0.520$ & $0.445$ & $0.100$ & $0.420$ & $0.350$ & $0.157$ & $0.000$ & $0.495$ & $0.480$ & $0.253$ & $0.033$ \\
      & \texttt{TTRL} & $0.625$ & $0.560$ & $0.469$ & $0.066$ & $0.483$ & $0.460$ & $0.193$ & $0.000$ & $\boldsymbol{0.510}$ & $0.490$ & $\boldsymbol{0.313}$ & $\underline{0.167}$ \\
      & \texttt{RoiRL} $g_I$ & $\boldsymbol{0.712}$ & $\boldsymbol{0.690}^\star$ & $\boldsymbol{0.518}^\star$ & $\underline{0.133}$ & $\boldsymbol{0.720}^\star$ & $\boldsymbol{0.680}^\star$ & $\boldsymbol{0.421}^\star$ & $\boldsymbol{0.067}^\star$ & $\underline{0.508}$ & $\underline{0.520}$ & $\boldsymbol{0.313}$ & $\boldsymbol{0.200}^\star$ \\
      & \texttt{RoiRL} $g_\beta$ & $\underline{0.685}$ & $\underline{0.650}$ & $\underline{0.469}$ & $\boldsymbol{0.200}$ & $\underline{0.543}$ & $\underline{0.560}$ & $\underline{0.277}$ & $\underline{0.033}$ & $\underline{0.508}$ & $\boldsymbol{0.530}^\star$ & $0.229$ & $0.100$ \\
    \midrule
    
    \multirow{1}{*}{$\texttt{maj}_{128}$} 
      & \texttt{Base} ($\pi_0$) & $0.717$ & $0.680$ & $0.506$ & $0.233$ & $0.563$ & $0.560$ & $0.289$ & $0.000$ & $0.543$ & $0.520$ & $0.361$ & $0.167$ \\
    \bottomrule
  \end{tabular}}
\end{table*}

{\bf Experimental Setup.} We design our experimental setting to compare the training and generalization performance of \texttt{RoiRL} against \texttt{TTRL}. We evaluate the learning approaches on three mathematical reasoning benchmarks: MATH500 \cite{math500}, AMC \cite{math500}, and AIME2024 \cite{aime2024}. The MATH500 dataset is further divided into 400 training (MATH500 \emph{Train}) and 100 test (MATH500 \emph{Test}) problems. All training algorithms are run on the unlabeled problems from the \emph{Train} split of MATH500, which defines the problem dataset $\mathcal{P}_n$. Using ground-truth labels, we then measure accuracy on the MATH500 \emph{Train} split, the MATH500 \emph{Test} split, AMC, and AIME2024 to assess the generalization capabilities of the learning methods. For base models, we use three reasoning-oriented LLMs of diverse sizes: Qwen2.5-Math-1.5B \cite{qwen}, Phi-4-mini-reasoning-4B \cite{phi4}, and Llama-3.2-3B-Instruct \cite{llama}. These models already demonstrate good reasoning capabilities and differ sufficiently in design to enable robust validation of our learning approaches across architectures and training paradigms \cite{spurious}



{\bf Training.} Both \texttt{TTRL} and \texttt{RoiRL} use the majority vote signal $\tilde{r}_k$ as a training signal to improve the base model. For each problem $x_i$, we generate $k = 10$ candidates, with which the majority vote signal $\tilde{r}_k$ is defined. \texttt{TTRL} is implemented using GRPO \cite{grpo} with the KL regularizer $\beta = 0.1$, that we compare to two flavors of \texttt{RoiRL}: the first one uses an exponential function $g_\beta: x\rightarrow \exp(x/\beta)$ with $\beta = 0.1$ to mimic \texttt{TTRL}'s behavior, and the second uses the identity function $g_I: x\rightarrow x$, reducing the offline update to simple \emph{supervised finetuning} on good generated answers. Implementation details are developed in Appendix~\ref{app:exp_details}.



{\bf Decoding strategies and baselines.} We evaluate our learning methods based on their ability to improve the base model $\pi_0$. Since training uses the majority-vote signal with $k = 10$, improving a single sampled answer ($k = 1$) is relatively straightforward, as the model effectively distills the majority vote. To assess performance beyond this distilled signal, we compare models using $k = 1$ ($\texttt{maj}_{1}$) and $k = 10$ ($\texttt{maj}_{10}$). We also report the base model with $k = 128$ ($\texttt{maj}_{128}$) to evaluate whether the learned model can surpass this strong but costly baseline.



{\bf Results.} Table~\ref{tab:results} shows that \texttt{RoiRL} outperforms the online RL-based \texttt{TTRL} approach in the majority of cases while being up to $2.5\times$ faster (see Appendix~\ref{computational}). \texttt{RoiRL} with $g_\beta$ improves over \texttt{TTRL} despite both targeting a KL-regularized objective, while the $g_I$ variant achieves the best overall results, suggesting that alternative reward transforms beyond KL may be more effective. \texttt{RoiRL} is a self-improving method and not merely a distillation of majority voting: $\texttt{maj}_{1}$ decoding with obtained models can surpass the base model’s $\texttt{maj}_{10}$ decoding, and $\texttt{maj}_{10}$ decoding with obtained models can outperform the base model with costly $\texttt{maj}_{128}$ decoding, even on unseen problems. Extended results with training curves and developed discussions are provided in Appendix~\ref{sec:extended_empirical_results}.

\section{Conclusion}

We proposed \texttt{RoiRL}, a simple and scalable approach for self-supervised reasoning in LLMs that converts majority-vote signals into efficient offline updates. Unlike online RL approaches such as \texttt{TTRL}, \texttt{RoiRL} requires no reference model, achieves greater stability and speed, and consistently improves accuracy across benchmarks, demonstrating that lightweight offline reinforcement learning is sufficient for self-improvement in reasoning tasks. Future work will extend this approach by evaluating \texttt{RoiRL} on larger LLMs to further validate scalability, exploring alternative ground-truth estimation strategies beyond majority vote, and studying the impact of different reward transforms, which we found to substantially influence performance.

\newpage

\bibliography{bibfile}
\bibliographystyle{plain}

\newpage


\appendix




\section{Technical Discussions and Proofs}

\subsection{Useful Lemma}

Our main theoretical result is based on the use of the following Lemma.

\begin{lemma}\label{lemma}
Let $g$ a positive function. The solution of the weighted likelihood objective of the form:
\begin{align*}
    \argmax_\pi \left\{\mathbb{E}_{ (c,y) \sim \pi_0(\cdot|x)} \left[g(x,c,y) \log \pi(c,y|x) \right] \right\}\,
\end{align*}
can be computed analytically and is:
$$\forall x, (y, c)\,, \pi^\star(c,y|x) \propto g(x,c,y) \pi_0(c,y|x)\,.$$
\end{lemma}
\begin{proof}
We prove this Lemma below for completeness. The objective is a constrained optimization problem, that can be solved by Lagrange multipliers. For a fixed input $x$, we want to maximize:
$$J(\pi) = \sum_{c,y} \pi_0(c,y|x) g(x,c,y) \log \pi(c,y|x)$$
subject to the normalization constraint:
$$\sum_{c,y} \pi(c,y|x) = 1$$

Setting up the Lagrangian:
$$\mathcal{L}(\pi, \lambda) = \sum_{c,y} \pi_0(c,y|x) g(x,c,y) \log \pi(c,y|x) - \lambda\left(\sum_{c,y} \pi(c,y|x) - 1\right)$$

Taking the partial derivative with respect to $\pi(c,y|x)$ and setting it to zero:
$$\frac{\partial \mathcal{L}}{\partial \pi(c,y|x)} = \frac{\pi_0(c,y|x) g(x,c,y)}{\pi(c,y|x)} - \lambda = 0$$

Solving for $\pi(c,y|x)$:
$$\pi(c,y|x) = \frac{\pi_0(c,y|x) g(x,c,y)}{\lambda}$$

Using the normalization constraint $\sum_{c,y} \pi(c,y|x) = 1$:
$$\sum_{c,y} \frac{\pi_0(c,y|x) g(x,c,y)}{\lambda} = 1$$

Therefore:
$$\lambda = \sum_{c,y} \pi_0(c,y|x) g(x,c,y)$$

Substituting back:
$$\pi^\star(c,y|x) = \frac{\pi_0(c,y|x) g(x,c,y)}{\sum_{c',y'} \pi_0(c',y'|x) g(x,c',y')}$$

This shows that $\pi^\star(c,y|x) \propto g(x,c,y) \pi_0(c,y|x)$, completing the proof.
\end{proof}

\subsection{Analytical Solution of \texttt{RoiRL}} \label{app:solution}

We remind the reader of the \texttt{RoiRL} algorithm below:

\RestyleAlgo{ruled}
\renewcommand{\thealgocf}{}
\begin{algorithm}
\caption{\textbf{R}easoning with \textbf{o}ffline \textbf{i}terative \textbf{R}einforcement \textbf{L}earning (\texttt{RoiRL})}
\textbf{Input}: Policy $\pi_\theta$, transforms $(g_m)_{m\in\mathbb{N}}$, prompt dataset $\mathcal{P} = \{x_i\}_{i \in [n]}$, number of candidates $k$. \\
\textbf{Initialize}: $\theta=\theta_0$, $\pi_0=\pi_{\theta_0}$ \\
\For{$m=1,2,\dots$}{
Construct offline dataset $\mathcal{D}_{m-1}$ with $\pi_{m-1}$ \\
Using $\mathcal{D}_{m-1}$, we approximate and solve the weighted log-likelihood objective
\begin{align*}
    \theta_{m} &= \argmax_\theta \left\{\sum_{i = 1}^n\mathbb{E}_{ (c,y) \sim \pi_{m-1}(\cdot|x_i)} \left[g_m(\tilde{r}_k(y,x_i, \theta_{m-1})) \log \pi_\theta(c,y|x_i) \right] \right\}\,,
\end{align*}
Set $\pi_m \leftarrow \pi_{\theta_m}$
}
\end{algorithm}

\texttt{RoiRL} leverages a set $(g_m)_{m\in\mathbb{N}}$ of increasing reward transform $g_m: \mathbb{R} \to \mathbb{R}\,, \forall m$ that implicitly control the reward influence. At iteration $m$, the analytical solution of Algorithm~\ref{alg:RoiRL} takes the form:
\begin{align*}
    \forall (c,y), x_i\,, \quad \pi_m(c,y|x_i) \propto \left(\prod_{j =1}^m g_j(\tilde{r}_k(y,x_i, \theta_{j-1}))\right) \pi_0(c,y|x_i)\,.
\end{align*}

\begin{proof}
We prove this result by induction on the iteration number $m$.

\textbf{Base Case ($m = 1$):}
At iteration $m = 1$, we solve:
\begin{align*}
    \theta_1 = \argmax_\theta \left\{\sum_{i = 1}^n\mathbb{E}_{ (c,y) \sim \pi_{0}(\cdot|x_i)} \left[g_1(\tilde{r}_k(y,x_i, \theta_{0})) \log \pi_\theta(c,y|x_i) \right] \right\}
\end{align*}

As the optimization problem is decomposable, we can look at each $x_i$ independently. Let $x_i$ be a prompt from $\mathcal{P}$. By the previous Lemma~\ref{lemma}, the analytical solution is:
\begin{align*}
    \pi_1(c,y|x_i) \propto g_1(\tilde{r}_k(y,x_i, \theta_{0})) \pi_0(c,y|x_i)
\end{align*}

This matches our claimed form with $m = 1$:
\begin{align*}
    \pi_1(c,y|x_i) \propto \left(\prod_{j=1}^1 g_j(\tilde{r}_k(y,x_i, \theta_{j-1}))\right) \pi_0(c,y|x_i) = g_1(\tilde{r}_k(y,x_i, \theta_{0})) \pi_0(c,y|x_i)
\end{align*}

\textbf{Inductive Step:}
Assume the claim holds for some iteration $m-1$, i.e.:
\begin{align*}
    \pi_{m-1}(c,y|x_i) \propto \left(\prod_{j=1}^{m-1} g_j(\tilde{r}_k(y,x_i, \theta_{j-1}))\right) \pi_0(c,y|x_i)
\end{align*}

At iteration $m$, we solve:
\begin{align*}
    \theta_m = \argmax_\theta \left\{\sum_{i = 1}^n\mathbb{E}_{ (c,y) \sim \pi_{m-1}(\cdot|x_i)} \left[g_m(\tilde{r}_k(y,x_i, \theta_{m-1})) \log \pi_\theta(c,y|x_i) \right] \right\}
\end{align*}

By the previous lemma, the analytical solution is:
\begin{align*}
    \pi_m(c,y|x_i) &\propto g_m(\tilde{r}_k(y,x_i, \theta_{m-1})) \pi_{m-1}(c,y|x_i)
\end{align*}

Substituting the inductive hypothesis:
\begin{align*}
    \pi_m(c,y|x_i) &\propto g_m(\tilde{r}_k(y,x_i, \theta_{m-1})) \left(\prod_{j=1}^{m-1} g_j(\tilde{r}_k(y,x_i, \theta_{j-1}))\right) \pi_0(c,y|x_i) \\
    &\propto \left(\prod_{j=1}^{m} g_j(\tilde{r}_k(y,x_i, \theta_{j-1}))\right) \pi_0(c,y|x_i)
\end{align*}

This completes the induction and proves the claimed form.
\end{proof}

\subsection{\texttt{RoiRL} solves the \texttt{TTRL} objective and beyond} \label{app:equivalence}

The proposed \texttt{RoiRL} objective provides an offline, iterative alternative to the recently proposed \texttt{TTRL} algorithm. We recall that \texttt{TTRL} optimizes the KL regularized expected reward:
\begin{align*}
    \max_\pi \left\{\sum_{i = 1}^n \mathbb{E}_{(c, y)  \sim \pi(\cdot|x_i)} \left[\tilde{r}_k(y, x_i, \pi) \right] - \beta \operatorname{KL}(\pi, \pi_0|x_i) \right\}\,.
\end{align*}
This optimization problem differs from the classical regularized objective as the individual rewards depend themselves on the current policy we are optimizing. We can connect \texttt{RoiRL} and \texttt{TTRL} with the following proposition:

\prop* 

\begin{proof}

Let us focus on the \texttt{TTRL} objective. As the problem is decomposable over prompts, the optimal policy $\pi^\star$  can be recovered for each $x_i$. We then optimize:
\begin{align*}
    J_i(\pi) = \sum_{c,y} \pi(c,y|x_i) \tilde{r}_k(y, x_i, \pi) - \beta \sum_{c,y} \pi(c,y|x_i) \log \frac{\pi(c,y|x_i)}{\pi_0(c,y|x_i)}
\end{align*}

Using the method of Lagrange multipliers with the constraint $\sum_{c,y} \pi(c,y|x_i) = 1$:
\begin{align*}
    \mathcal{L} = J_i(\pi) - \lambda_i \left(\sum_{c,y} \pi(c,y|x_i) - 1\right)
\end{align*}

Taking the functional derivative with respect to $\pi(c,y|x_i)$:
\begin{align*}
    \frac{\partial \mathcal{L}}{\partial \pi(c,y|x_i)} &= \tilde{r}_k(y, x_i, \pi) + \sum_{c',y'} \pi(c',y'|x_i) \frac{\partial \tilde{r}_k(y', x_i, \pi)}{\partial \pi(c,y|x_i)} \\
    &\quad - \beta \log \frac{\pi(c,y|x_i)}{\pi_0(c,y|x_i)} - \beta - \lambda_i = 0
\end{align*}

The key challenge is the second term, which captures how changing $\pi(c,y|x_i)$ affects all other rewards $\tilde{r}_k(y', x_i, \pi)$ through the policy dependence.

For the optimal policy $\pi^\star$, rearranging:
\begin{align*}
    \beta \log \frac{\pi^\star(c,y|x_i)}{\pi_0(c,y|x_i)} &= \tilde{r}_k(y, x_i, \pi^\star) + \sum_{c',y'} \pi^\star(c',y'|x_i) \frac{\partial \tilde{r}_k(y', x_i, \pi^\star)}{\partial \pi^\star(c,y|x_i)} - \beta - \lambda_i
\end{align*}

Taking the exponential:
\begin{align*}
    \pi^\star(c,y|x_i) &= \pi_0(c,y|x_i) \exp\left(\frac{1}{\beta}\left[\tilde{r}_k(y, x_i, \pi^\star) + \sum_{c',y'} \pi^\star(c',y'|x_i) \frac{\partial \tilde{r}_k(y', x_i, \pi^\star)}{\partial \pi^\star(c,y|x_i)} - \beta - \lambda_i\right]\right)
\end{align*}

Using the normalization constraint to determine $\lambda_i$, the first-order conditions become:

\begin{align*}
    \pi^\star(c,y|x_i) &\propto \pi_0(c,y|x_i) \exp\left(\frac{1}{\beta}\left[\tilde{r}_k(y, x_i, \pi^\star) + \sum_{c',y'} \pi^\star(c',y'|x_i) \frac{\partial \tilde{r}_k(y', x_i, \pi^\star)}{\partial \pi^\star(c,y|x_i)}\right]\right)\,.
\end{align*}
Finally, as the rewards are indicator functions, which is discontinuous, their derivative is null almost surely. This allows us to set all the rewards partial derivative to $0$, obtaining a fixed point equation that  $\pi^\star$ verifies:
\begin{align}
    \forall (c,y)\,, \pi^\star(c,y|x) \propto \exp\left(\frac{1}{\beta} \tilde{r}_k(y, x_i, \pi^\star) \right)\pi_0(c,y|x_i)\,.
\end{align}
\texttt{TTRL} targets this solution by solving Equation~\ref{eq:kl_regularised}. However, obtaining this solution can also be conducted by solving the fixed point equation directly. The fixed point solution can be solved by iterating over $m \ge 1$ the following:
\begin{itemize}
    \item Collect a dataset $\mathcal{D}_{m-1}$ with $\pi_{m-1}$.
    \item Update $\pi_m \propto \exp\left(\frac{1}{\beta} \tilde{r}_k(y, x_i, \pi_{m-1}) \right)\pi_0(c,y|x_i) $ 
\end{itemize}
until you reach convergence of the policy $\pi_m$. The update step can be solved by optimization, and can be implemented by the following Algorithm:

\RestyleAlgo{ruled}
\begin{restatable}{algorithm}{alg2}
\caption{Fixed Point Approach}\label{alg:fp}
\textbf{Input}: Policy $\pi_\theta$, prompt dataset $\mathcal{P} = \{x_i\}_{i \in [n]}$, number of candidates $k$. \\
\textbf{Initialize}  $\pi_0$. \\
\For{$m=1,2,\dots$}{
Construct offline dataset $\mathcal{D}_{m-1}$ with $\pi_{m-1}$ \\
Set $b_{m-1}(y,x_i) = \tilde{r}_k(y,x_i, \pi_{m-2})$ if $m > 2$ else $b_{m-1}(y,x_i) = 0$. \\
Using $\mathcal{D}_{m-1}$, we approximate and solve the weighted log-likelihood objective
\begin{align}
    \pi_m &= \argmax_\pi \left\{\sum_{i = 1}^n\mathbb{E}_{\pi_{m-1}(\cdot|x_i)} \left[\exp\left(\frac{1}{\beta}(\tilde{r}_k(y,x_i, \theta_{m-1}) - b_{m-1}(y,x_i))\right) \log \pi_\theta(c,y|x_i) \right] \right\}\,.
\end{align}
}
\end{restatable}
This algorithm is exactly Algorithm~\ref{alg:RoiRL} with the particular choice of $g_m$ to be:
\begin{align}\label{eq:g_ttrl}
    g_m(y,x) = \exp\left(\frac{1}{\beta} \left(\tilde{r}_k(y,x, \pi_{m-1}) - b_m(x,y) \right) \right)\,,
\end{align}
with $b_{m-1}(y,x) = \tilde{r}_k(y,x, \pi_{m-2})$ if $m > 2$ else $b_{m-1}(y,x) = 0$. \texttt{RoiRL} can exactly target the optimal solution of \texttt{TTRL} with the choice of $g$ in Equation~\ref{eq:g_ttrl}. This ends the proof.

\end{proof}

\section{Extended Empirical Results}
\label{sec:extended_empirical_results}

\subsection{Experimental Details}

\subsubsection{Implementation Details} \label{app:exp_details}

We compared three training methods using the same hyperparameters. In all experiments, we generate $k = 10$ candidates for each problem and then train on these candidates using the chosen reward function $g$. We used our custom \verb|WeightedSFTTrainer| to implement learning with $g_\beta : x \rightarrow \exp(x / \beta)$ and used standard \verb|SFTTrainer| to implement $g_I : x \rightarrow x$, as it is equivalent to using supervised finetuning on candidates with answers corresponding to majority vote. We will refer to the proccess of generating candidates and training on them as 1 round of \texttt{RoiRL}. In both methods, after generating candidates, we train for $3$ epochs during every round. We trained all methods for $15$ rounds, with early stopping if $\texttt{maj}_{10}$ accuracy on the train dataset did not improve for more than $5$ rounds. We take the round with the best performance on train, when reporting the final accuracies.

All hyper-parameters were set to their default values in \verb|SFTTrainer|, the only exception being the reduction of the learning rate from $2\cdot10^{-5}$ to $10^{-6}$ for Llama-3.2 with $g_I$ because higher values result in overfitting. We implemented \texttt{TTRL} with Huggingface \verb|GRPOTrainer|  and used its default parameters.

In our experiments, we compare one round of \texttt{RoiRL} with one epoch of \texttt{TTRL} as both require the same number of generations, namely, $k n$. However, note that \texttt{TTRL} is computationally more demanding and thus requires more time per epoch. More details of the computational advantages of \texttt{RoiRL} over \texttt{TTRL} are discussed in \ref{computational}.

\subsubsection{Evaluation} To obtain the majority vote, we generate $k$ answers from the model using temperature sampling with a temperature equal to $1.0$. The maximum number of new tokens is set to 1024. We then extract the answers from the generated solution, find the one that has the most occurrences, and choose it as our majority vote answer. If two answers have equal number of occurrences, we pick one of them randomly.

To extract the answer, we find the first occurrences of \verb|\boxed{}| in the generated solution and consider everything that was put inside the brackets as the final answer. Qwen2.5-Math automatically puts the final answer inside \verb|\boxed{}| and to ensure the same behavior from the other two models, we add a phrase ``\verb|Put your answer in \boxed{}|'' to the prompt.

Since the same answer can be written in multiple ways (for example, $0.5$ or $\frac{1}{2}$) we used \verb|sympy| to parse latex and then compared the answers as sympy objects. In case of parsing or comparison errors or timeouts, we fallback and compared the answers as strings.

\subsubsection{Computational advantages of \texttt{RoiRL}}
\label{computational}

The proposed \texttt{RoiRL} method has several computational advantages over \texttt{TTRL}. Firstly, we strictly separate the generation and the training phases and this allows better batching during generation. More precisely, we can generate answers to multiple questions in a single batch, unlike online RL algorithms such as \texttt{TTRL}. Secondly, our reward method does not require storing logits, so we can use larger batch size during generation compared to \texttt{TTRL}. Unlike GRPO, our method does not use a reference model to compute the reward, further reducing its computational cost. And finally, using a sparse reward function (e.g., $g_I$) in \texttt{RoiRL} can significantly speed up the training phase, especially in the early stages, when the majority answers are sparse. With all these advantages combined, we achieve performance more than $\times 2.5$ times faster per 1 round of \texttt{RoiRL} compared to $1$ epoch of \texttt{TTRL}. The exact time evaluations for \texttt{TTRL} and \texttt{RoiRL} with sparse ($g_I$) and dense ($g_\beta$) reward functions are presented in Table~\ref{tab:time_compr}.

All experiments were conducted on a Google Cloud Platform (GCP) instance with a single NVIDIA A100 (80Gb VRAM), 12 vCPUs and 170 GB RAM running on Debian GNU/Linux 11 (bullseye). \texttt{RoiRL} and \texttt{TTRL} methods were implemented using the Hugging Face Transformers library v4.52.4, TRL v0.18.2 and PyTorch v2.7.1 with CUDA 12.4.

\begin{table}[h!]
  \caption{Time per round comparison}
  \label{tab:time_compr}
  \centering
  \begin{tabular}{ll}
    \toprule
    Method & Time per round \\
    \midrule
    \texttt{RoiRL}, $g_I$ (sparse reward) & 6552.5s \\
    \texttt{RoiRL}, $g_\beta$ (dense reward) & 8883.5s \\
    \texttt{TTRL} & 17019.25s \\
    \bottomrule
  \end{tabular}
  \vspace{-0.25cm}
\end{table}

\subsection{Detailed Results and Discussions}

Figures \ref{img:plot_qwen}, \ref{img:plot_phi} and \ref{img:plot_llama} illustrate the training curves for Qwen2.5-Math-1.5B \cite{qwen}, Phi-4-mini-reasoning-4B \cite{phi4}, and Llama-3.2-3B-Instruct \cite{llama} trained with \texttt{RoiRL} and \texttt{TTRL}. The greedy decoding baseline and the $\texttt{maj}_{128}$ baseline are represented by horizontal lines.

\texttt{RoiRL} consistently improves the accuracy on the MATH500-train dataset and successfully generalizes on a holdout MATH500-test dataset, AMC and AIME datasets. Note that unlike supervised finetuning, \texttt{RoiRL} does not require ground-truth labels even on the train dataset, so this is a self-improvement process. Compared to \texttt{TTRL}, \texttt{RoiRL} demonstrates faster convergence with the same number of training rounds and less computations.

\texttt{RoiRL} improves not only the sampling accuracy (dotted lines), but also the $\texttt{maj}_{10}$ accuracy (solid lines). Moreover, for Qwen-2.5 and Phi-4, after several epochs the sampling accuracy exceeds the initial $\texttt{maj}_{10}$ accuracy. This demonstrates how \texttt{RoiRL} does not just distill the majority vote performance into the base model, but improves the general ability to solve mathematical problems.

\begin{figure}[!ht]
    \centering
    \includegraphics[width=1\textwidth]{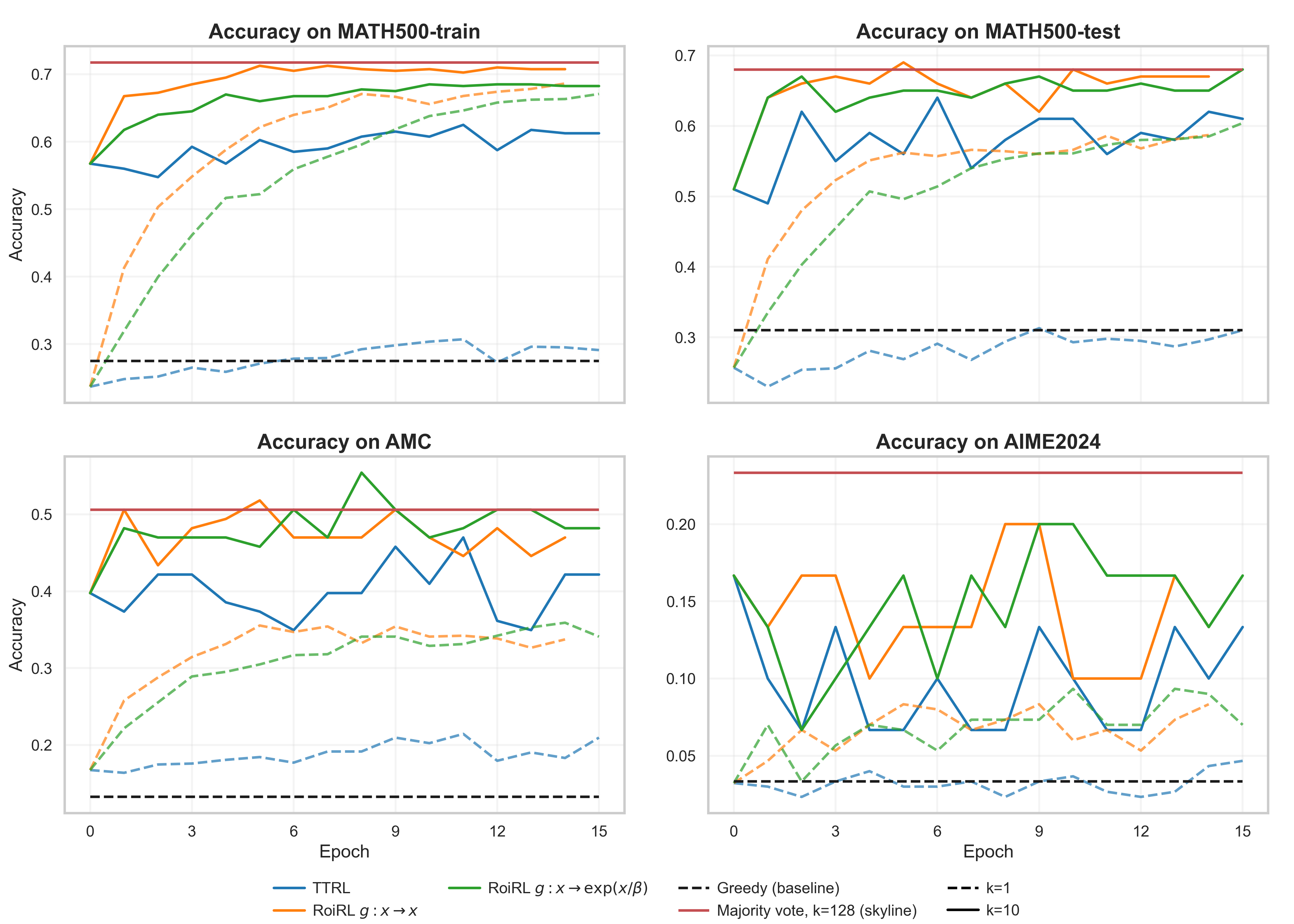}
    \caption{Training curves for Qwen-2.5-Math}
    \label{img:plot_qwen}
\end{figure}

\begin{figure}[!ht]
    \centering
    \includegraphics[width=1\textwidth]{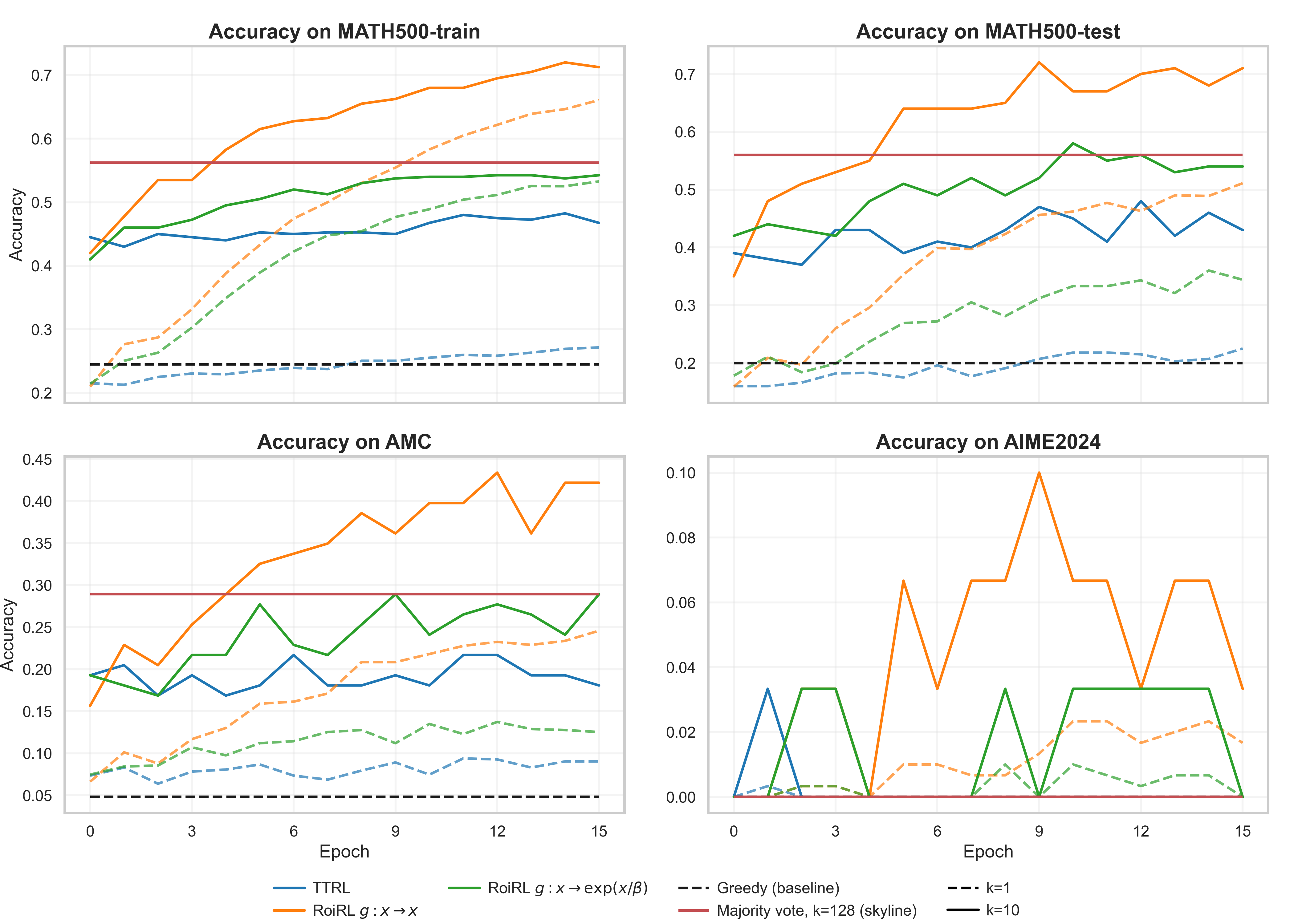}
    \caption{Training curves for Phi-4}
    \label{img:plot_phi}
\end{figure}

\begin{figure}[!ht]
    \centering
    \includegraphics[width=1\textwidth]{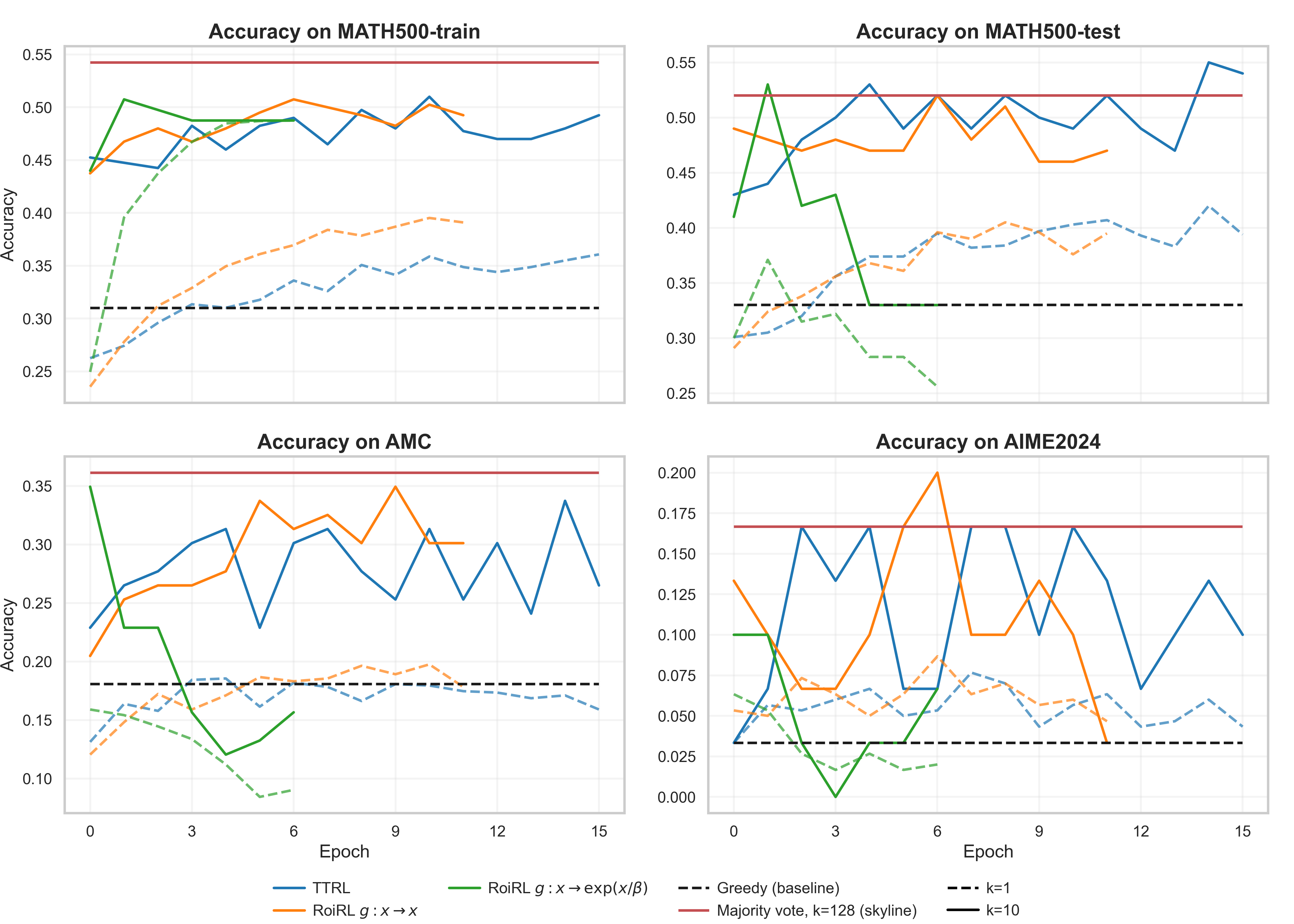}
    \caption{Training curves for Llama-3.2}
    \label{img:plot_llama}
\end{figure}



Finally, Figure \ref{img:entropies} illustrates the evolution of the average entropy during training for Qwen2.5 on MATH500-Train and Test. When using the \texttt{RoiRL}, entropy rapidly decreases to almost zero. However, for \texttt{TTRL}, the entropy stays relatively high during the whole training process. This can explain faster convergence of \texttt{RoiRL} during our experiments. In addition, the fast reduction of the entropy to zero with \texttt{RoiRL} raises a natural question of applying more regularization, implicitly by reducing the learning rate or using alternative reward functions,  which may be the subject of further research.

\begin{figure}[!ht]
    \centering
    \includegraphics[width=1\textwidth]{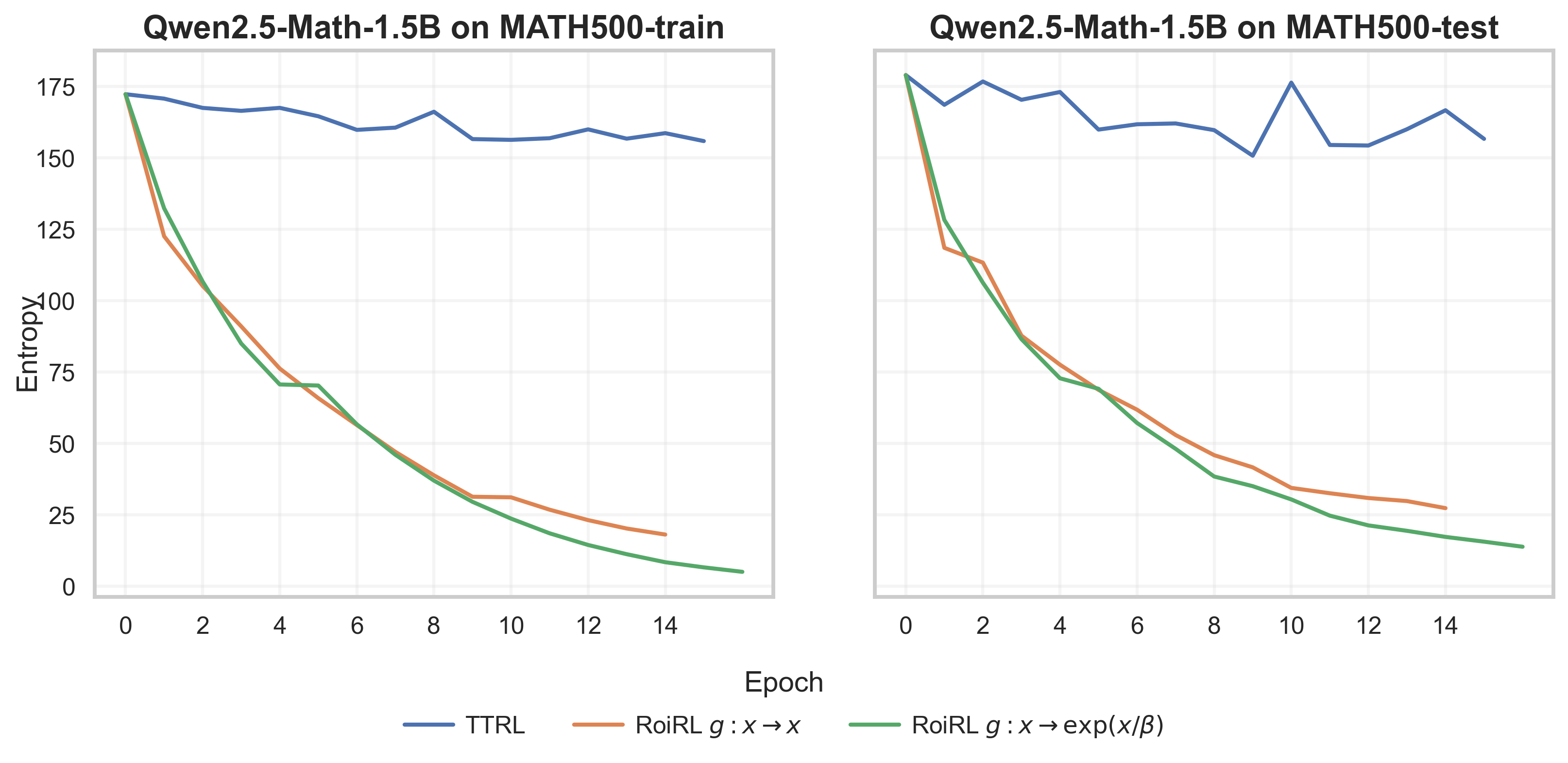}
    \caption{Entropies for Qwen2.5 on MATH500}
    \label{img:entropies}
\end{figure}

\end{document}